\documentclass[letterpaper, 10 pt, conference]{ieeeconf}  % Comment this line out if you need a4paper

\IEEEoverridecommandlockouts                              % This command is only needed if 
\usepackage[binary-units=true,per-mode=symbol]{siunitx}
\usepackage{amsmath,amssymb}
\usepackage{graphicx}
\usepackage{color}
\usepackage{xspace}
\usepackage{balance}
\usepackage{paralist}
\usepackage{booktabs}
\usepackage{tikz}
\usepackage{url}

\graphicspath{{./figures/}}
\newcommand{\ie}{i.e.,\ }
\newcommand{\eg}{e.g.,\ }
\newcommand{\wrt}{w.r.t.\xspace}

\def\atantwo{\mathop{\rm atan2}}
\def\sgn{\mathop{\rm sgn}}
\newcommand{\norm}[1]{\left\lVert#1\right\rVert}

\newlength{\figureheight}
\newcommand{\etal}{\xspace{}et al.\xspace}

\newcommand{\reffig}[1]{Fig.~\ref{#1}}
\newcommand{\reftab}[1]{Tab.~\ref{#1}}

\newtheorem{corollary}{Corollary}
\usepackage{cite}

\usepackage{fancyhdr}
\title{\LARGE \bf
Search-based 3D Planning and Trajectory Optimization for Safe\\Micro Aerial Vehicle Flight Under Sensor Visibility Constraints
}

\author{Matthias Nieuwenhuisen$^{1}$ and Sven Behnke$^{2}$% <-this % stops a space
\thanks{This work was partially funded by the German Research Foundation (DFG) under grant BE 2556/8-2.}% <-this % stops a space
\thanks{$^{1}$M. Nieuwenhuisen is with the Fraunhofer Institute for Communication, Information Processing and Ergonomics FKIE, Wachtberg, Germany.
  {\tt\small nieuwenh@ais.uni-bonn.de}}%
\thanks{$^{2}$S. Behnke is with the Autonomous Intelligent Systems Group, Computer Science VI, University of Bonn, Germany.}%
}

\begin{document}

\maketitle
\thispagestyle{fancy}
\pagestyle{empty}

%%%%%%%%%%%%%%%%%%%%%%%%%%%%%%%%%%%%%%%%%%%%%%%%%%%%%%%%%%%%%%%%%%%%%%%%%%%%%%%%
\begin{abstract}
Safe navigation of Micro Aerial Vehicles (MAVs) requires not only obstacle-free flight paths according to a static environment map, but also the perception of and reaction to previously unknown and dynamic objects.
This implies that the onboard sensors cover the current flight direction.
Due to the limited payload of MAVs, full sensor coverage of the environment has to be traded off with flight time.
Thus, often only a part of the environment is covered.

We present a combined allocentric complete planning and trajectory optimization approach taking these sensor visibility constraints into account.
The optimized trajectories yield flight paths within the apex angle of a Velodyne Puck Lite 3D laser scanner enabling low-level collision avoidance to perceive obstacles in the flight direction.
Furthermore, the optimized trajectories take the flight dynamics into account and contain the velocities and accelerations along the path.

We evaluate our approach with a DJI Matrice 600 MAV and in simulation employing hardware-in-the-loop.
\end{abstract}

%%%%%%%%%%%%%%%%%%%%%%%%%%%%%%%%%%%%%%%%%%%%%%%%%%%%%%%%%%%%%%%%%%%%%%%%%%%%%%%%
\section{Introduction}
The environments in which micro aerial vehicles (MAVs) operate become more challenging with new applications, \eg indoor and disaster response operations.
These scenarios prohibit the optimistic assumption that direct flight paths are obstacle-free at a certain altitude that can be reached.
Furthermore, the assumption that the environment is static cannot be made in the presence of human or machine activities, or when the structural integrity of a building cannot be assured.
Hence, continuous monitoring of the environment that is traversed and quick reaction to unforeseen obstacles is key to safe and collision-free flights.
In our own previous work~\cite{nieuwenhuisen2015jint}, we have presented a system that follows planned paths incorporating a potential field-based safety layer that avoids unknown obstacles based on a laser-based 3D map acquired with \SI{2}{\hertz}~\cite{nieuwenhuisen2013isprs}.

To increase the safe flight speed, a faster perception of the environment for localization and obstacle avoidance is inevitable.
Modern lightweight 3D laser scanners as the Velodyne Puck Lite acquire 3D point clouds of the environment at a high frequency.
Nevertheless, this comes at a price: The new laser scanner setup does only cover a vertical field-of-view (FoV) of \SI{30}{\degree} in contrast to the \SI{180}{\degree} of our old omnidirectional laser scanner.
This raises the requirement for alternate paths that let the MAV only move within the FoV of the scanner to reliably detect obstacles that impose a collision hazard.  
Similar considerations have to be taken into account when using camera or radar sensors.
These sensors furthermore require that the MAV is only flying into forward direction of the sensor.

\begin{figure}[t]
  \includegraphics[width=\linewidth]{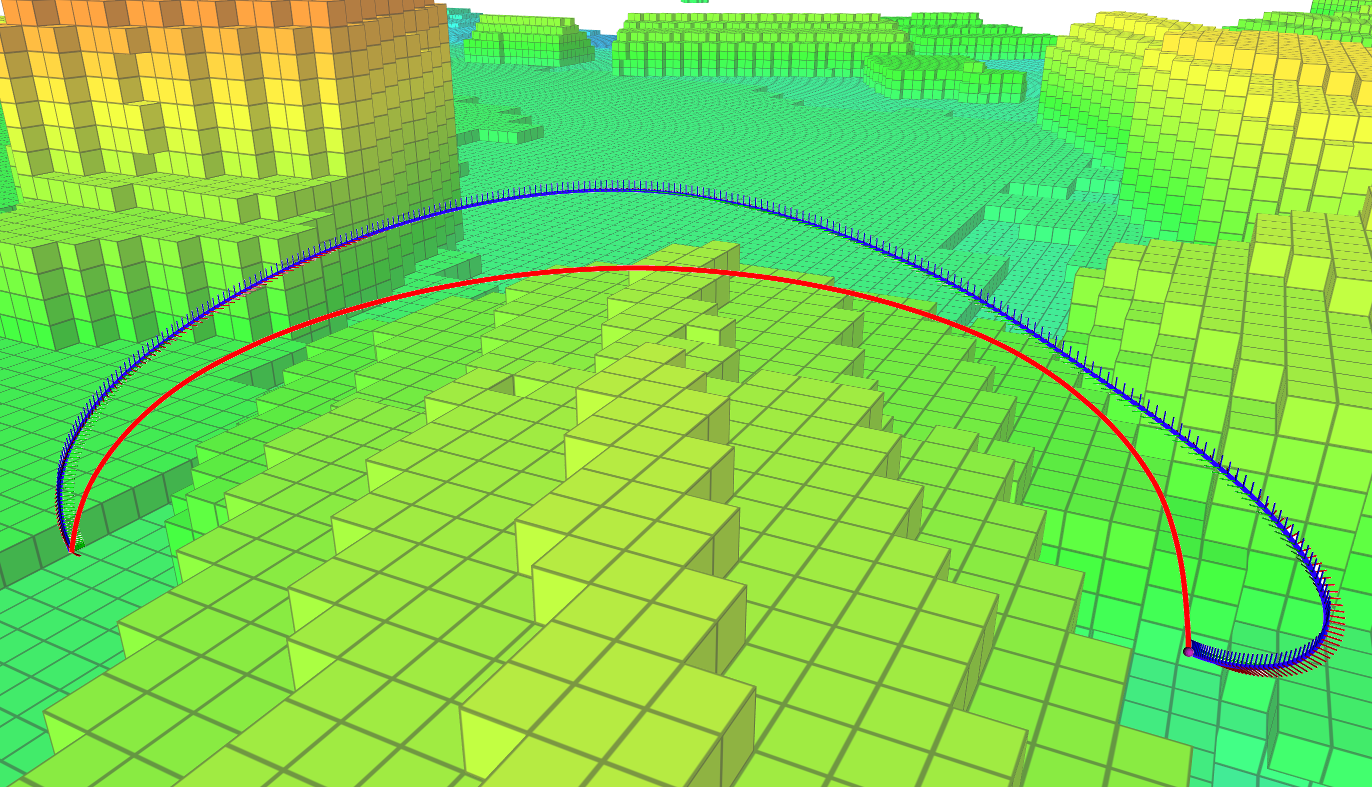}
  \caption{Optimized trajectories without (red) and with visibility constraints (blue). The constrained trajectory lunges out to keep the ascent and descent angles within the field of view of the onboard sensors. The flight starts on the left.}
  \label{fig:teaser}
  \vspace{-3ex}
\end{figure}

One option would be to define motion patterns for ascent and descent that ensure the perception of the flight path and to plan at fixed altitudes in-between.
This yields far-from-optimal flight paths, though.
We follow a two-layered approach to generate allocentric obstacle-free paths.
An allocentric path planner finds obstacle-free paths in a global map that respects the visibility constraints locally.
This path is further refined to a dynamically feasible trajectory in a second step---preserving the sensor coverage property.
\reffig{fig:teaser} shows resulting trajectories with and without visibility constraints after the optimization step.
To reach the target position close to a building, our approach generates a spiraling descent.
The resulting trajectories contain velocity and acceleration information employed by our low-level controller to accurately follow the intended paths.

To show the applicability of our approach, we analyze it qualitatively and evaluate it with a hardware-in-the-loop simulator resembling a DJI Matrice 600 MAV and with the real MAV in an outdoor scenario.

Our main contributions are:
\begin{compactitem}
  \item a search-based planning representation that takes visibility constraints into account for either image-based sensors or sensors covering an inverse double-conic volume,
  \item an accordingly adapted heuristic that speeds up informed path planners as variants of A* and D*, and
  \item a trajectory optimization that refines the planned paths to smooth, dynamically feasible trajectories.
\end{compactitem}

\section{Related Work}
\label{sec:related-work}
To plan high-dimensional trajectories, often sampling-based planners are employed, including KPIECE~\cite{csucan2009kinodynamic} and randomized kinodynamic planning~\cite{lavalle2001randomized}.
In addition to those sampling-based motion planning algorithms, trajectory optimization allows for efficient generation of high-dimensional trajectories.
Covariant Hamiltonian Optimization and Motion Planning (CHOMP) is a gradient-based optimization algorithm proposed by Ratliff et al.~\cite{Zucker_2013_7421}.
It uses trajectory samples, which initially can include collisions, and performs a covariant gradient descent by means of a differentiable cost function to find an already smooth and collision-free trajectory. 
A modified version of the stochastic optimizer STOMP~\cite{stomp} has been used for multicriteria optimization~\cite{pavlichenko2017iros}.
The optimized criteria include, in addition to obstacle costs, the trajectory duration and joint limits, but no sensor visibility constraints.
Another algorithm derived from CHOMP is ITOMP, an incremental trajectory optimization algorithm for real-time replanning in dynamic environments~\cite{itomp}.
In order to consider dynamic obstacles, conservative bounds around them are computed by predicting their velocity and future position.
Since fixed timings for the trajectory waypoints are employed and replanning is done within a time budget, generated trajectories may not always be collision-free.

Augugliaro \etal~\cite{augugliaro2012generation} compute collision-free trajectories for multiple MAVs simultaneously. Obstacles other than the MAVs are not considered.
Similar to our approach, Richter \etal~\cite{richter2013polynomial} plan MAV trajectories in a low dimensional space (using RRT*) and optimize the trajectory with a dynamics model afterwards to achieve short planning times.
Our approach does not have the constraint that the optimized path has to include the planned waypoints.
Another approach using optimization by means of polynomial splines between waypoints focuses on time-optimal trajectories computed in real-time~(Bipin \etal~\cite{bipin2014autonomous}).
Collisions are avoided by intermediate waypoints from a high-level planner and are not explicitly considered in the optimization process.
Andreasson \etal~\cite{drivethedrive} employ optimization to compute steerable trajectories for automated ground vehicles. 
Oleynikova et al.~\cite{oleynikova2016} optimize trajectories with continuous timings by employing polynomials.
Fang et al.~\cite{fang2017} add a global planning layer to initialize trajectory optimization---similar to our approach, but their planning layer is restricted to 2D. In contrast, we plan in 3D space with visibility constraints.

Majumdar and Tedrake~\cite{majumdar2016} use compositions of preprocessed trajectories to generate flight paths that are safe under uncertainty in real-time.
In contrast, we frequently reoptimize a trajectory in real-time to react on changes in the environment and uncertain path execution.
Zhang et al.~\cite{zhang2018iros} generate a set of dynamically feasible paths prior to a flight and quickly select suitable ones based on sensor input during the flight.
They do not consider sensor visibility constraints.
Richter and Roy~\cite{richter2016learning} plan trajectories for wheeled robots that reduce unknown space in the direction of travel to achieve higher safe velocities.

Many approaches address the problem of planning sensor poses, \eg \cite{englot2010iros,stefas2018icra,dang2018}.
In contrast to our work, they aim at covering allocentric areas of interest not necessarily in the direction of flight.
We aim at covering egocentric areas of interest that move together with the MAV.
Complementary to our approach is the planning for configuring the sensor FoV to cover a safety volume based on a given flight path~\cite{arora2015}.

\section{Problem Specification}
For MAV navigation, we aim at smooth and safe trajectories.
Smoothness allows for fast, continuous trajectory following by the low-level controllers without unnecessary stopping.
Trajectories are safe if they stay in a safe distance to known obstacles and if unknown obstacles can be reliably perceived by obstacle avoidance sensors.
The first objective can be achieved with trajectory optimization \wrt the vehicle dynamics---in our case, we optimize for low acceleration costs.
Trajectory safety is ensured by adding visibility constraints and obstacle costs to our objective function:
The trajectory ascent and descent angles should stay within the vertical FoV of our obstacle sensor.
To avoid unfeasible local minima when optimizing the trajectory, we follow a two-tier approach:
First, we plan an obstacle-free path incorporating the visibility constraints by graph-search which is complete and optimal \wrt the planning specification.
In the second step, we optimize this planned trajectory with CHOMP~\cite{Zucker_2013_7421} for a tailored objective function.

\section{Path Planning}
For initial path planning, we employ A* graph-search on a representation based on a modified regular grid.
The general case of finding obstacle-free shortest paths is straightforward with a cost function modeling the distance from graph nodes to nearest obstacles.
We extend this approach to visibility-constrained planning by modifying the planning representation and adapting the employed heuristic.
For the case of a vertical apex angle $\phi$ of the obstacle sensors smaller than \SI{90}{\degree}, the angular resolution of a uniform voxel grid of \SI{45}{\degree} is too coarse to represent the allowed maximum ascent and descent angles of $\frac{\phi}{2}$.
In our case, the apex angle of the lidar is \SI{30}{\degree} requiring an angular resolution of \SI{15}{\degree}.
To increase the angular resolution, we employ an anisotropic voxel grid with horizontal edge lengths of $v_{xy}$ and a voxel height of $v_z = \tan(\frac{\phi}{2}) v_{xy}$.
From the 26 edges connecting nodes centered in the voxels of the grid with their Moore neighborhood, we remove the two edges connecting voxels directly above or below the current voxel.
The resulting graph structure enforces restricted ascents/descents within the FoV of the sensors.

To penalize frequent changes in the flight direction, we introduce the direction of flight as additional planning dimension.
Without this penalty, a zigzag motion to ascent or descent would be equal to larger straight glide paths in path costs, but would significantly slow down the MAV due to numerous stops to change direction.
The direction of flight is discretized to the eight possible transitions in the plane, angles of up to \SI{45}{\degree} are not penalized.
We remove edges yielding larger changes in direction, thus, these transitions are still possible, but at the cost of multiple intermediate transitions.
\reffig{fig:constrained_planning} illustrates the resulting plans with and without visibility constraints.
The MAV orientation in the planned path depends on the mode.
If paths for MAVs with front-facing sensors, \eg cameras, are planned, the MAV orientation is coupled to the flight direction dimension.
Thus, the MAV yaw angle is linearly interpolated along plan edges such that the angle between the front of the MAV and the current flight direction is at most \SI{45}{\degree} and arrives at a difference of \SI{0}{\degree} when the next waypoint is reached.
For sensors with a horizontal FoV of at least \SI{90}{\degree}, the full path segment between waypoints is guaranteed to be visible.
Sensors with a smaller FoV require rotating the MAV in place until the path segment is in the sensor FoV before starting with the position interpolation.
In omnidirectional mode, the yaw orientation of the MAV can be freely set to mission requirements and the flight direction dimension is solely restricting sudden direction changes.

\begin{figure}[t]
  \centering
  \includegraphics[width=0.35\linewidth,clip,trim=0 140 0 50]{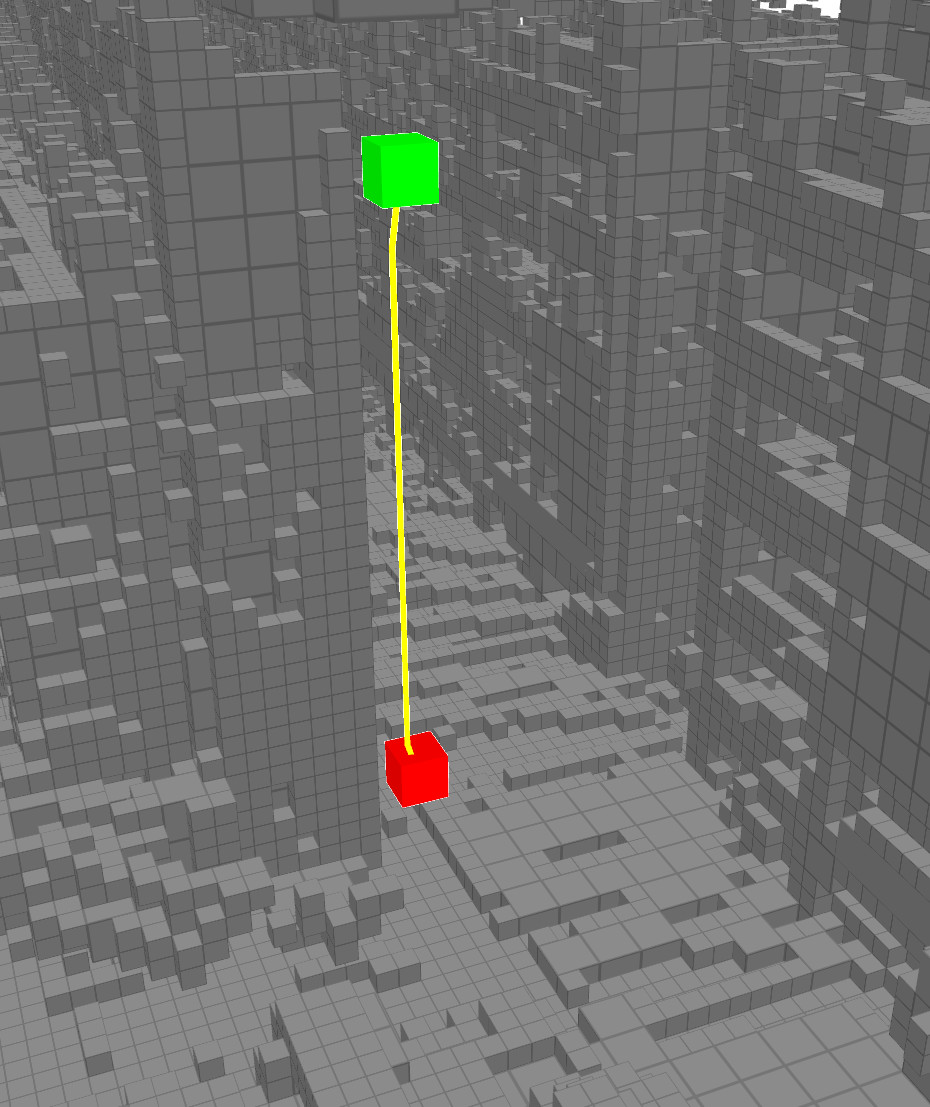}~
  \includegraphics[width=0.35\linewidth,clip,trim=0 140 0 50]{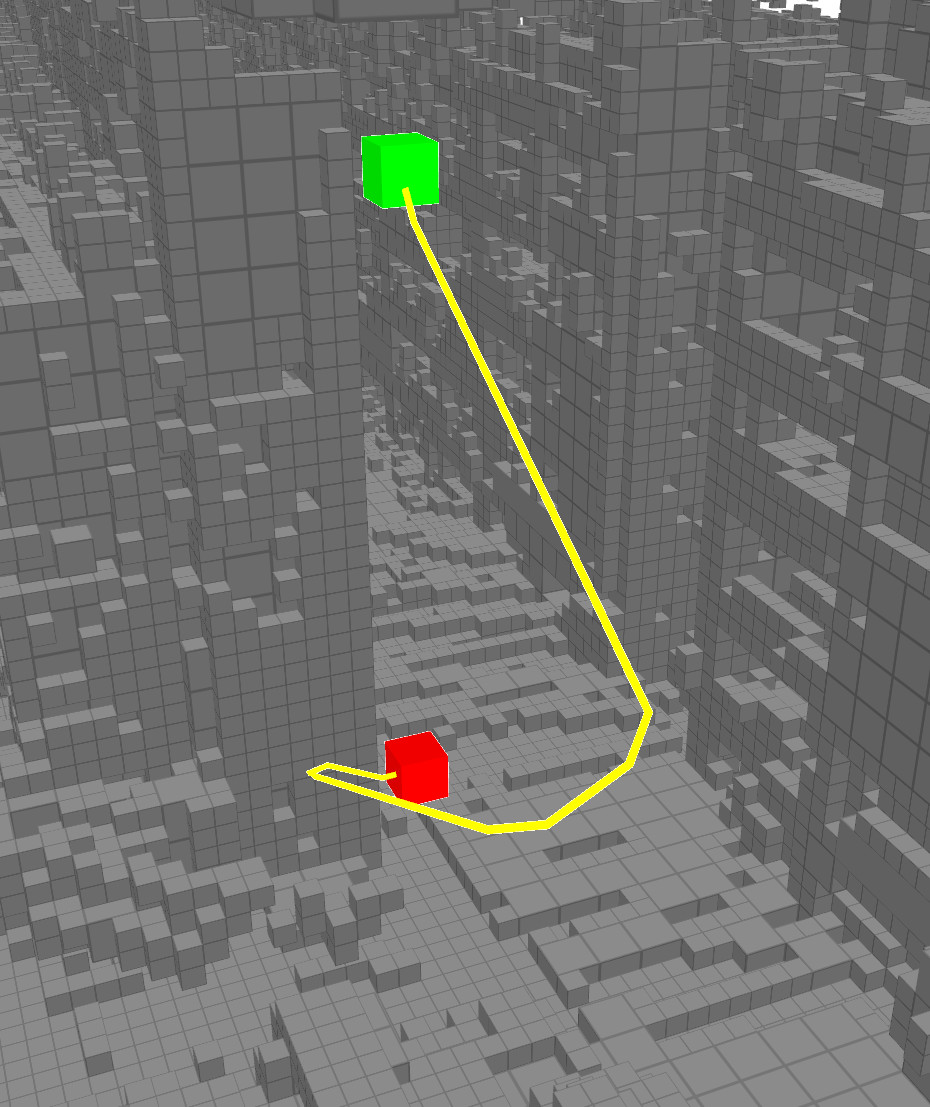}
  \vspace{-1ex}
  \caption{Planning under visibility constraints. Left: Without visibility constraints, the shortest path (yellow) from a start (green) to a target position (red) below solely descents in place. Right: With visibility constraints, the MAV has to move within the field of view of the lidar and consequently follows a longer descent path with an angle of \SI{15}{\degree}.}
  \label{fig:constrained_planning}
\end{figure}

\begin{figure}
  \centering
  \includegraphics[width=0.7\linewidth]{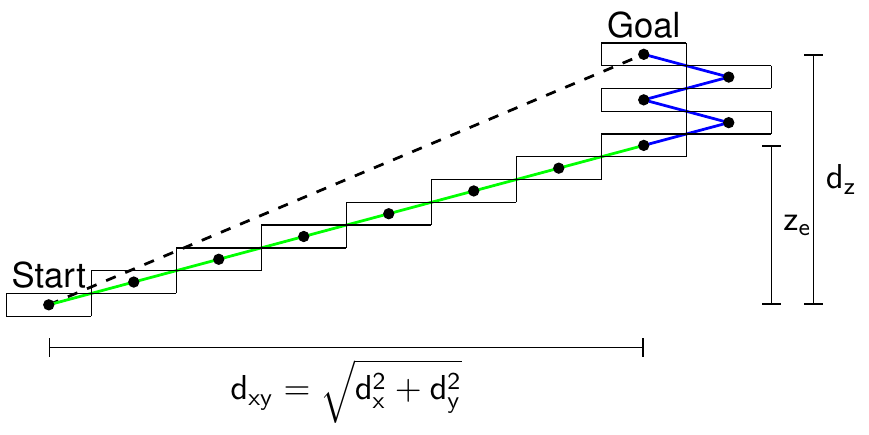}
  \vspace{-3ex}
  \caption{The Euclidean distance (dashed line) underestimates the path length in our planning representation.
  We split the heuristic into two parts: i) an Euclidean part (green) to the closest point to the goal the MAV can reach on a straight line given the sensor constraints; ii) an estimate for the shortest possible path for the remaining vertical distance to the goal $|d_z| - z_e$ (blue).}
  \label{fig:heuristic-scheme}
  \vspace{-4ex}
\end{figure}

As the Euclidean distance heuristic strongly underestimates altitude changes, we employ a modified heuristic better suited for our planning structure.
\reffig{fig:heuristic-scheme} illustrates the idea.
For a node position $p_n$ and a target position $p_t$, we define the heuristic $h(p_n - p_t) = h(d)$ on the position difference $d$ as

\begin{align}
  h(d) & = \sqrt{d_x^2 + d_y^2 + z_e^2} + z_z,\\
  z_e & = \min{(|d_z|,\tan{\frac{\phi}{2}}\sqrt{d_x^2+d_y^2})},\label{eq:heuristic_eucl}\\
  z_z & = \frac{\max{(0,|d_z|-z_e)}}{v_z} \sqrt{v_{xy}^2 + v_z^2},\label{eq:heuristic_zigzag}
\end{align}
where $z_e$ is the slope-restricted Euclidean altitude change that is possible over a distance $\sqrt{d_x^2 + d_y^2}$ with maximum angle $\phi / 2$;
$z_z$ is the shortest possible detour to overcome the remaining altitude difference.

\begin{corollary}
  The heuristic $h(.)$ is an admissible heuristic for A* search in the visibility constrained graph structure.
\end{corollary}
\begin{proof}
  In the first case
  \begin{equation}
  |d_z| \leq \tan{\left(\frac{\phi}{2}\right)}\sqrt{d_x^2+d_y^2}
  \end{equation}
  follows that $z_e = |d_z|$ and (\ref{eq:heuristic_zigzag}) vanishes.
  The remaining heuristic term
  \begin{equation}
    h(d) = \sqrt{d_x^2+d_y^2+d_z^2} = ||d||_2
  \end{equation}
  is then the Euclidean distance which is an admissible heuristic.

  In the other case, $z_e$ is the maximum altitude change that is possible with the allowed ascent angle, \ie $p_n + (d_x, d_y, z_e)$ is the closest point to $p_t$ that can be reached on a straight line from $p_n$ without violating the angular constraint.
  All points closer to the target $p_t$ in $z$ would increase the distance in the x-y-plane with factor $1 / \tan{(\phi / 2)}$ which is $>1$ following the assumption that $\phi \leq 90^{\circ}$ for this case and (\ref{eq:heuristic_eucl}).
  The remainder $z_r = |d_z|-z_e$ can only be eliminated by an ascent through $\frac{z_r}{v_z}$ voxels with edge length $\sqrt{v_{xy}^2 + v_z^2}$ each resulting in the value of $z_z$ which has to be added to the distance to the closest point.
  Thus, no shorter path exists and the heuristic is admissible in both cases.
\end{proof}

\begin{figure}
  \centering
  \setlength{\figureheight}{0.27\linewidth}
  \includegraphics[height=\figureheight]{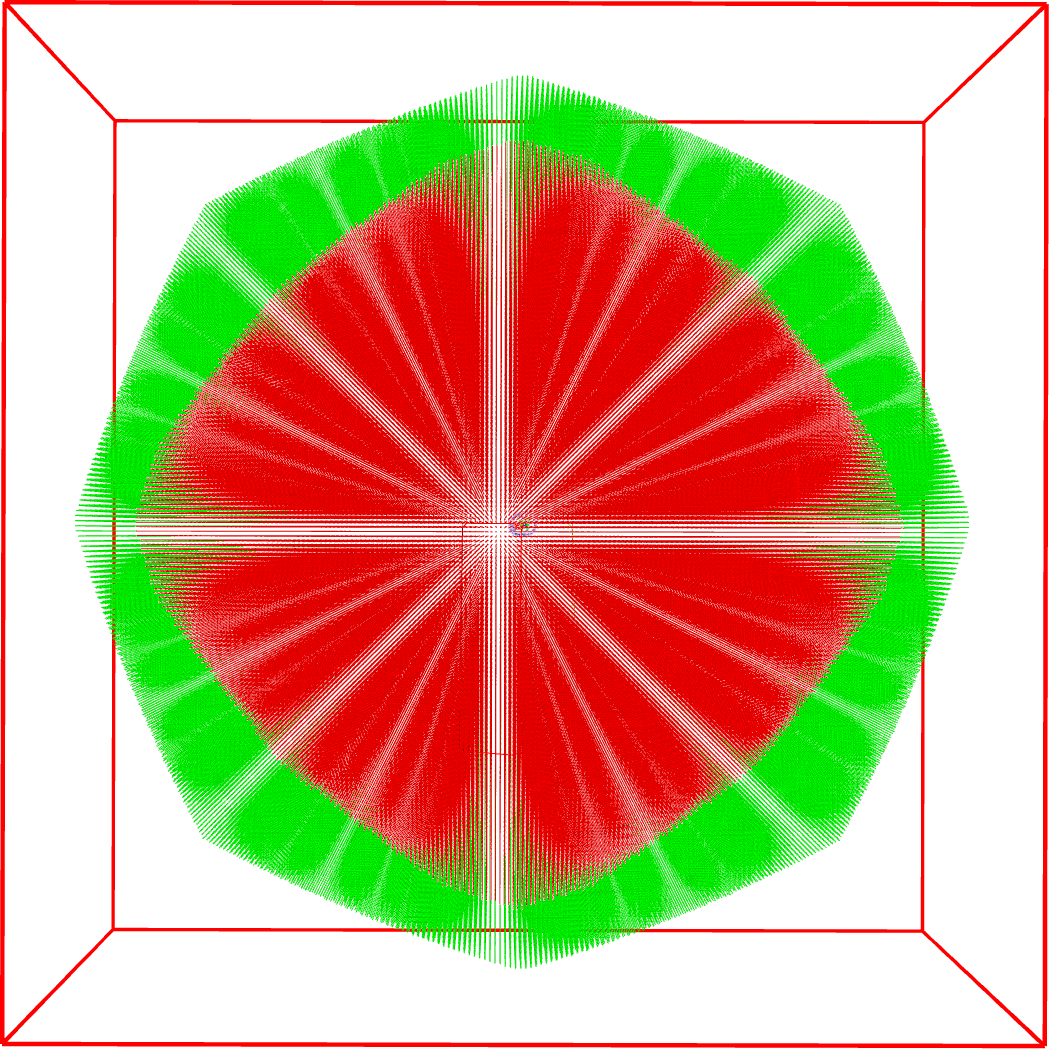}\,
  \includegraphics[height=\figureheight,trim=4cm 0 4cm 0,clip]{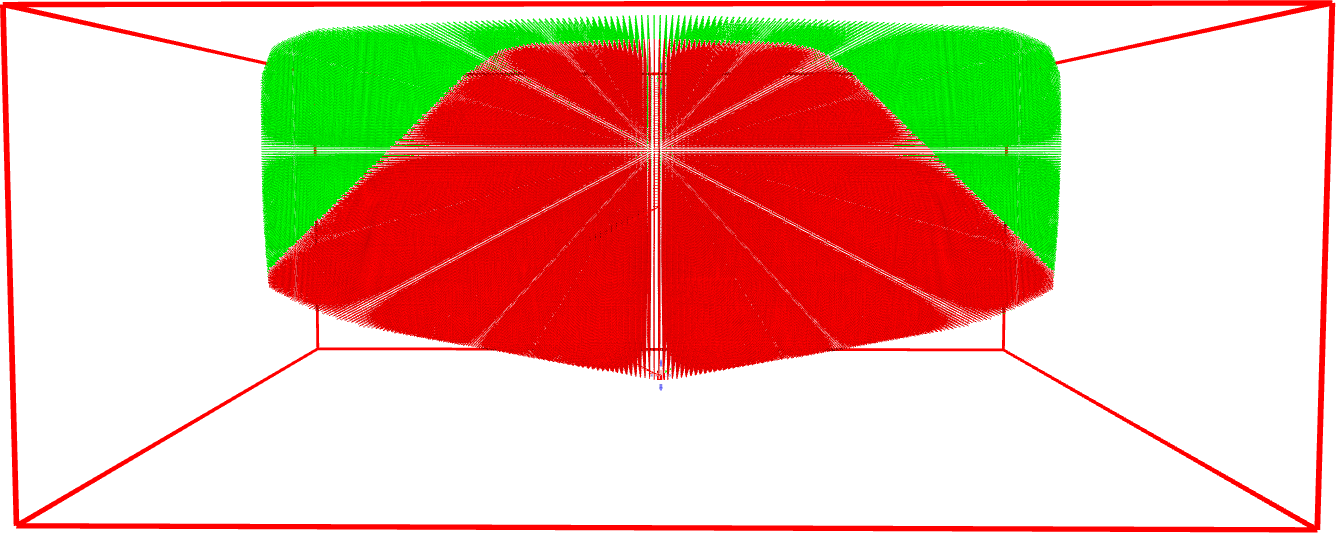}
  \caption{Visibility constraint planning heuristic. A standard Euclidean distance heuristic strongly underestimates the cost of altitude changes in the grid. This results in more unnecessary node expansions (green). Our modified heuristic expands fewer nodes (red). Left: Top-view. Right: Side-view. Red lines depict the planning volume.}
  \label{fig:heuristic}
  \vspace{-3ex}
\end{figure}

\reffig{fig:heuristic} shows the difference in node expansions with and without our modified heuristic for the path depicted in \reffig{fig:constrained_planning}

To speed up the node expansions without the requirement to process and store the whole graph structure in advance, our planner employs look up tables (LUT) for edge costs and possible angular transitions.
Furthermore, obstacle costs per vertex are cached in a lower dimension grid until they are invalidated by map updates.

To smooth the path, we replace parts of it with continuous curvature transition segments~\cite{nieuwenhuisen2016iros}.
The transition segments mitigate discontinuities in the derivatives of the path without violating obstacle constraints due to their convexity.
This yields dynamically smoother paths which are further refined in the following trajectory optimization step.

\section{Trajectory Optimization}
We re-discretize the path according to a simple analytical motion model with acceleration bounds to a \SI{10}{\hertz} time resolution.
The planned path is a timingless list of 4D (x, y, z, yaw) spatial waypoints.
To generate smooth trajectories for our MAV, we need poses and velocities as input for the underlying model predictive controller (MPC)~\cite{beul2017icuas}.
For accurate trajectory following, we have to optimize the trajectory for low acceleration control costs.
Consequently, outputs of our trajectory optimization step are time-discretized 12D trajectories with 4D poses, velocities, and accelerations without discontinuities.
Accordingly, the goal is to find a trajectory, which minimizes the costs calculated by a predefined cost function.
Similar to our existing approach~\cite{nieuwenhuisen2016iros}, the trajectory optimizer gets a start and a goal configuration $x_{0} = (p^x_0, p^y_0, p^z_0, \theta_0)^\top, x_{N} = (p^x_{N}, p^y_{N}, p^z_{N}, \theta_{N})^\top \in \mathbb{R}^4$ as input.
The output of the algorithm is a trajectory $\Theta \in \mathbb{R}^{4\times (N+1)}$ consisting of one trajectory vector $\Theta^d = (x^d_0, \dots, x^d_N)^\top \in \mathbb{R}^{N+1}$ per dimension $d$, discretized into $N+1$ waypoints.
The optimization problem we solve iteratively is defined by
\begin{equation}
  \min\limits_{\Theta}\left[\sum\limits^N_{i=0}q(\Theta_i) + \sum\limits_d\frac{1}{2}{\Theta^d}^\top\mathbf{R}\Theta^d\right].
\end{equation}

\begin{figure}[t]
  \centering
  \resizebox{0.8\linewidth}{!}{
    \input{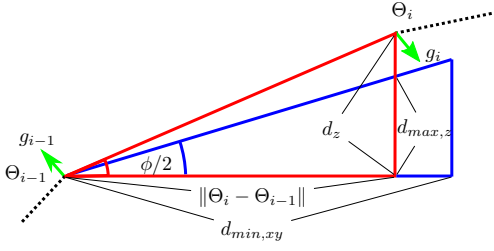}
  }
  \caption{If the ascent (or descent) angle between two consecutive waypoints $\Theta_{i-1}$ and $\Theta_i$ is out of the sensor FoV $\phi$, this case is depicted by the red triangle, then the altitude change $d_z$ is modified such that the constraint violation $d_z - d_{max,z}$ is reduced to half. The remaining violation is mitigated by stretching the planar projection of the movement $\Theta_{i,xy} - \Theta_{i-1,xy}$ to reduce the difference to $d_{min,xy}$. Thus, the trajectory becomes reshaped towards the blue triangle by the gradients $g_{i-1}, g_i$.}
  \label{fig:flattening}
  \vspace{-3ex}
\end{figure}

The state costs---obstacle costs, velocity and visibility constraints---are calculated by a predefined cost function $q(\Theta_i)$ for each state in $\Theta$.
${\Theta^d}^\top\mathbf{R}\Theta^d$ is the sum of control costs along the trajectory in dimension $d$ with $\mathbf{R}$ being a matrix representing control cost weighting.
The trajectory optimizer now attempts to solve the defined optimization problem by means of the gradient-based optimization method CHOMP~\cite{Zucker_2013_7421}.
The cost function $q(\Theta_i)$ is a weighted sum of I) piece-wise linear increasing costs $c_o$ induced by the proximity to obstacles, II) squared costs $c_a$ caused by acceleration limits, III) squared costs $c_v$ caused by velocity constraints, and IV) costs from violating visibility constraints.
The obstacle costs $c_o$ increase linearly with a slope $o_{\textrm{far}}$ from a maximum safety distance to an inflated minimum distance to the obstacle.
From the inflated minimum distance to the obstacle, costs increase with a steeper slope $o_{\textrm{close}}$ to allow for gradient computation in the vicinity of obstacles.

Velocities and accelerations as derivatives of the state are implicitly modeled by the duration between discretization steps.
The trajectory optimization converges faster when the initialization is close to the (locally) optimal trajectory. This includes velocities and accelerations.
Even though the optimal solution is naturally not known in advance, we can make some assumptions about the MAV dynamics that reduce the convergence time and avoid unfeasible local minima.
Thus, we initialize the trajectory optimizer with the planned path, which is optimal given the plan discretization and dimensionality.
We need to re-parameterize the planned path from a discrete-space to a discrete-time representation.
The number of resulting trajectory points and their distribution over time is estimated by an acceleration-bound simple motion model that can be calculated in closed-form~\cite{nieuwenhuisen2016iros}.

To enforce the visibility constraints, we look at the local path triangles defined by a segment between two trajectory points $\Theta_{i-1}$, $\Theta_{i}$ and its projection to the x-y-plane $\Theta_{i-1,xy}$, $\Theta_{i,xy}$.
Let $\Theta_{i-1}$ and $\Theta_i$ be two consecutive trajectory points in $\Theta^d$.
Then the visibility constraint for an sensor apex angle of $\phi$ is defined as
\begin{equation}
  \left|\atantwo{\left( \Theta_{i,z} - \Theta_{i-1,z}, \norm{ \Theta_{i,xy} - \Theta_{i-1,xy} } \right)}\right| \leq \frac{\phi}{2}.
\end{equation}

\begin{figure}[t]
  \centering
  \setlength{\figureheight}{0.4\linewidth}
  \includegraphics[height=\figureheight]{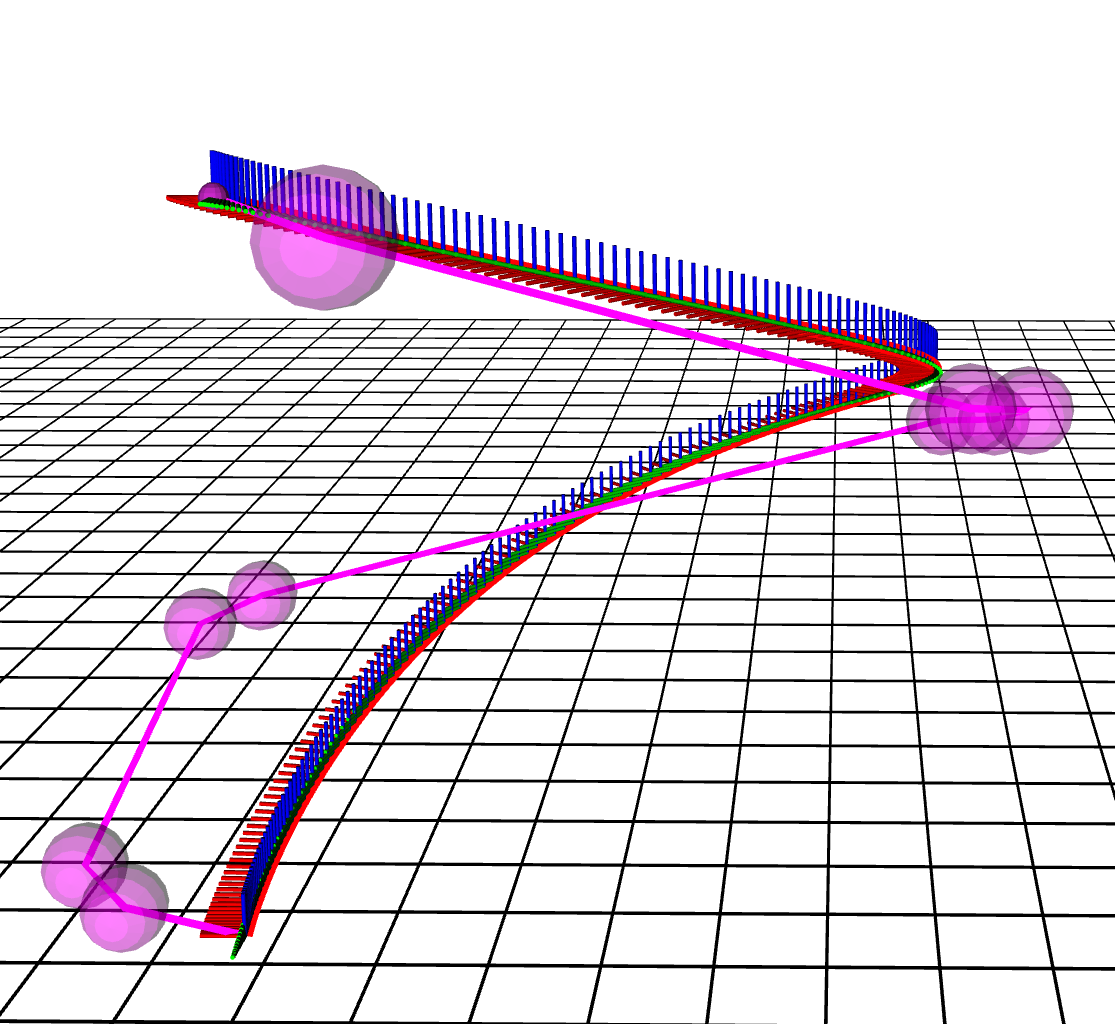}~
  \includegraphics[height=\figureheight]{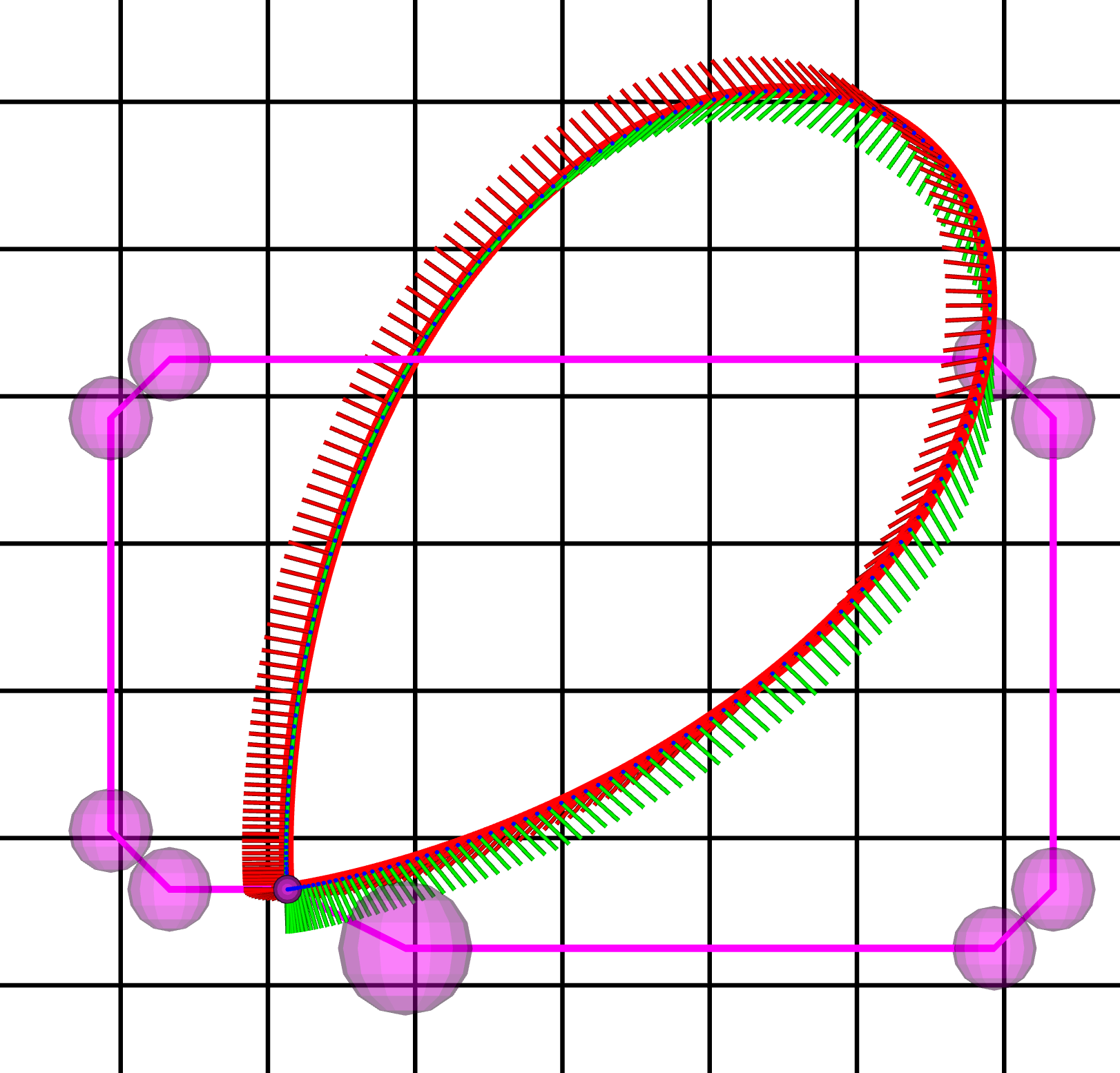}
  \caption{Optimized trajectory for an ascent in place. We initialize the trajectory optimization with a planned path (purple) with transition segments (purple spheres). The result after optimization yields a smooth spiral (colored axes). Left: Perspective. Right: Top-down ortho projection.}
  \label{fig:spiral_up}
  \vspace{-3ex}
\end{figure}

If this constraint is violated, we locally modify the trajectory points to flatten the path triangle.
Simultaneously, we reduce the altitude difference $d_z$ and stretch the movement in the x-y-plane, depicted in \reffig{fig:flattening}.
The partial gradients $g_{i-1}$ and $g_i$ to modify the trajectory points are defined as
\begin{align}
  d_{min,xy} &= \frac{\left| d_z \right| - d_{max,z}}{\tan(\phi / 2)} - \norm{ \Theta_{i,xy} - \Theta_{i-1,xy} }\\
  g_{i-1,x} &= w_v \cos( \alpha ) \frac{d_{min,xy}}{2}\\
  g_{i-1,y} &= w_v \sin( \alpha ) \frac{d_{min,xy}}{2}\\
  g_{i-1,z} &= w_v \sgn(-d_z) \frac{\left|d_z\right| - d_{max,z}}{4}\\
  g_{i} &= -g_{i-1},
\end{align}
where $w_v$ is a weighting factor and $\alpha$ is the direction angle of the path segment projected to the x-y-plane, 
$d_{min,xy}$ denotes the minimum planar distance to reach the angular constraint with a given $d_z$, and $d_{max,z}$ is the maximum allowed distance in z to reach the constraint with a given planar distance.
Thus, half of the constraint violation is distributed to the altitude gradients and the other half is used to elongate the path.
As a result, the optimized paths can lunge out to reduce the ascent/descent angles.
\reffig{fig:spiral_up} shows the resulting trajectory for an ascent in place.
Please note that the sensor visibility constraint is satisfied along the whole trajectory if it is satisfied in the discrete trajectory points by construction.

During flight, the trajectory is continuously re-optimized to cope with newly perceived obstacles.
The general approach is detailed in~\cite{nieuwenhuisen2016iros}.
With up to \SI{10}{\hertz}, the optimizer is initialized with the current remaining flight trajectory shortened by the estimated reoptimization duration.
The reoptimization duration estimate is based on the last duration as it is dominated by the remaining trajectory length which gets shorter during flight.
To account for small differences in the duration of single iterations, we add \SI{10}{\percent} overhead.
Finally, the reoptimized part of the trajectory is merged with the currently executed trajectory.

\section{Evaluation}
With our approach, the ascent and descent angles of the trajectories are bounded by the FoV of the onboard sensor.
When ascending or descending in place, as depicted in \reffig{fig:spiral_up}, the shortest path yields angles close to \SI{90}{\degree} for the whole flight, clearly not covered by the onboard sensor.
The resulting spiral motion after optimization facilitates a very smooth ascent with angles always close to the allowed maximum.

\begin{figure}
  \centering
  \includegraphics[width=0.8\linewidth]{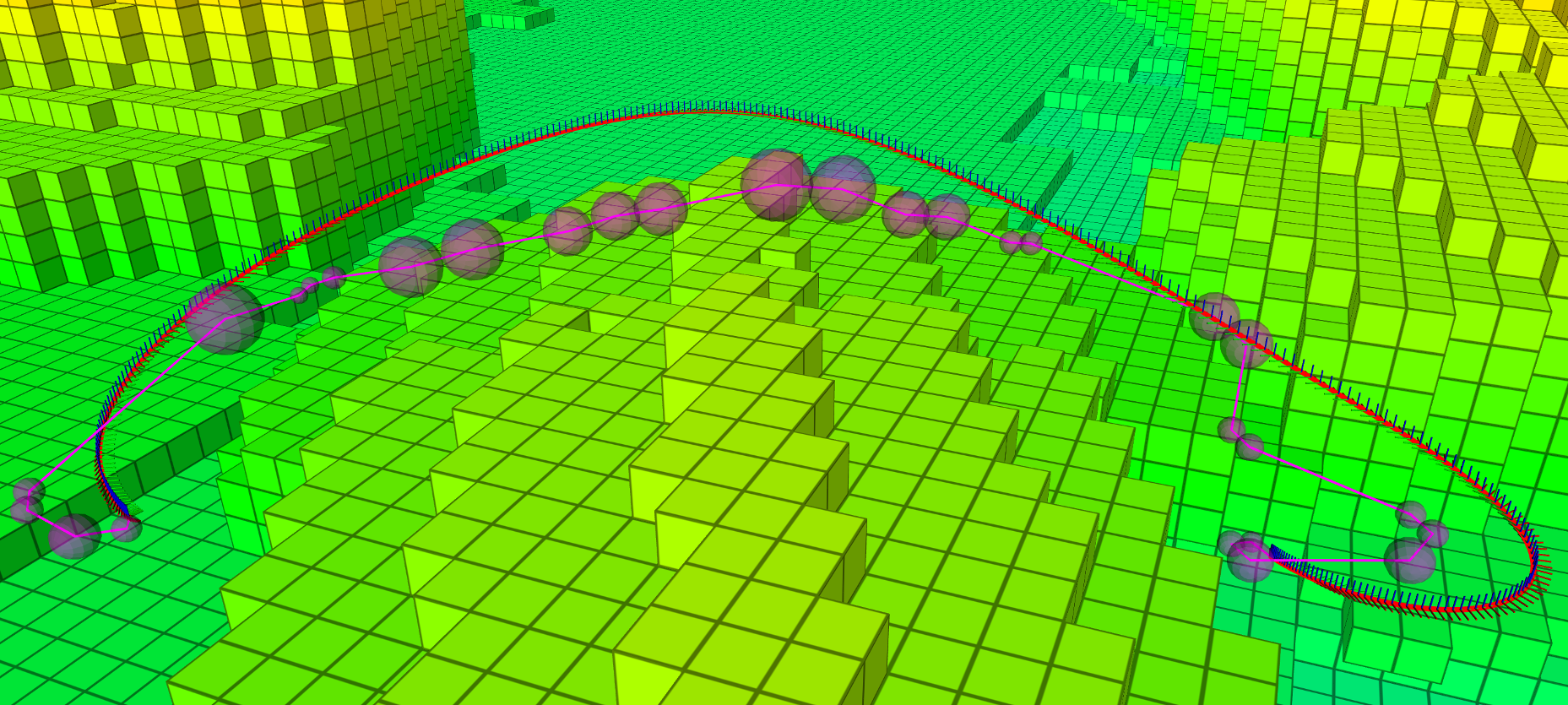}\\[0.03in]
  \includegraphics[width=0.8\linewidth]{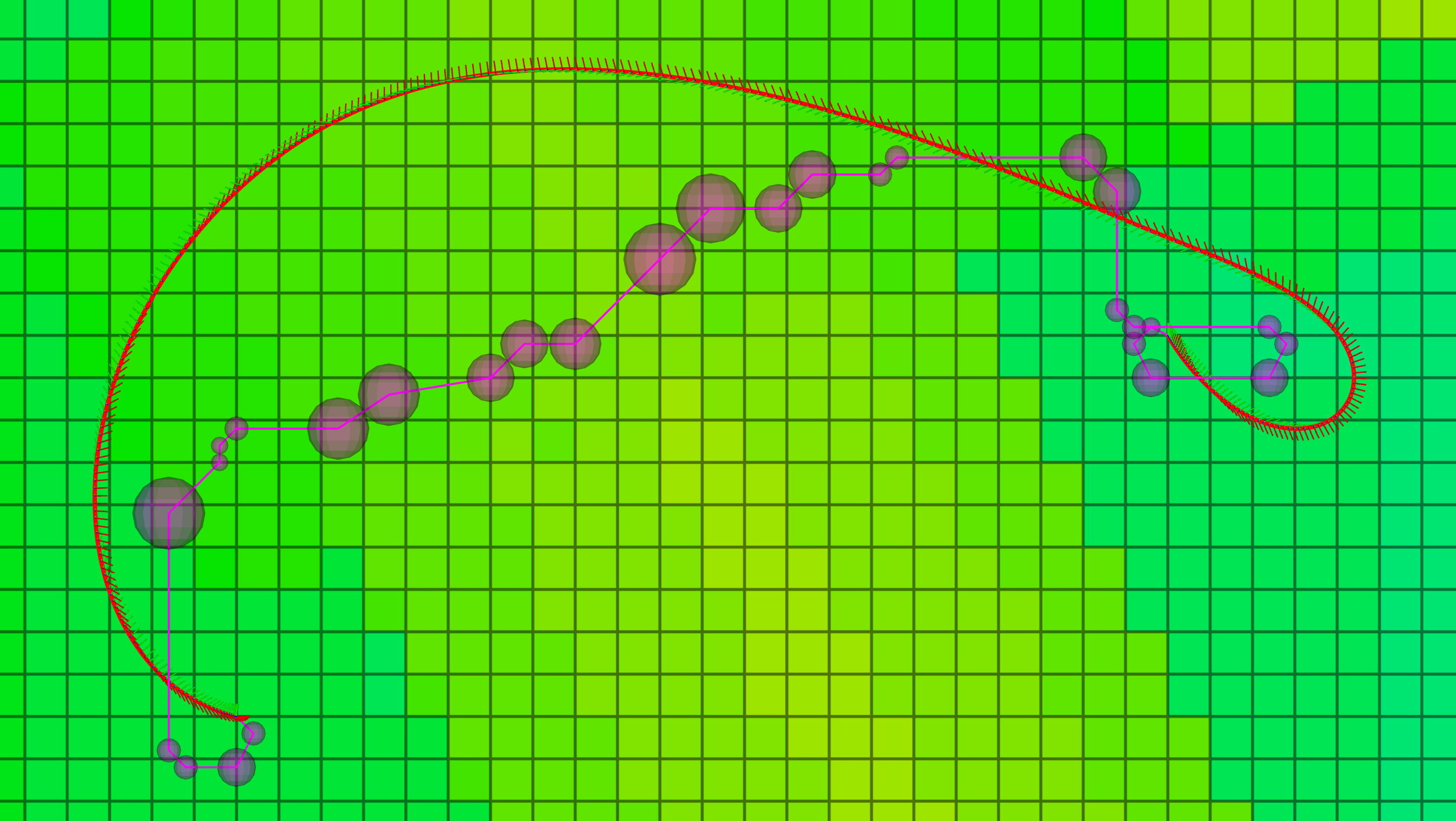}
  \vspace*{-1ex}
  \caption{Plan and trajectory in outdoor map. Whereas the planned path (purple) is more compact and shorter, the optimized trajectory allows higher velocities due to a smoother flight path.}
  \label{fig:frankenforst-long}
\end{figure}

\begin{figure}
  \setlength{\figureheight}{0.33\linewidth}
  \centering
  \includegraphics[height=\figureheight]{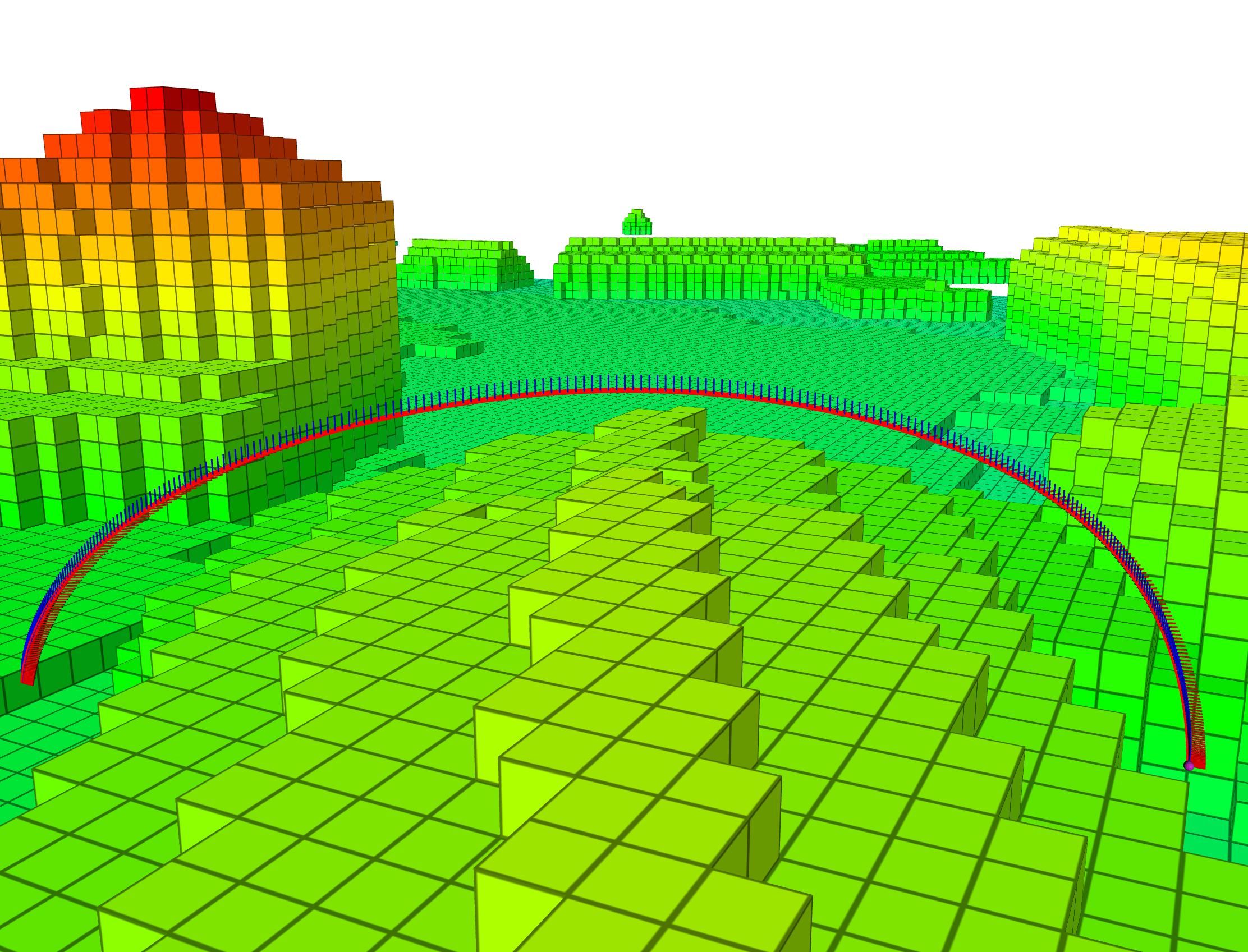}\,
  \includegraphics[height=\figureheight]{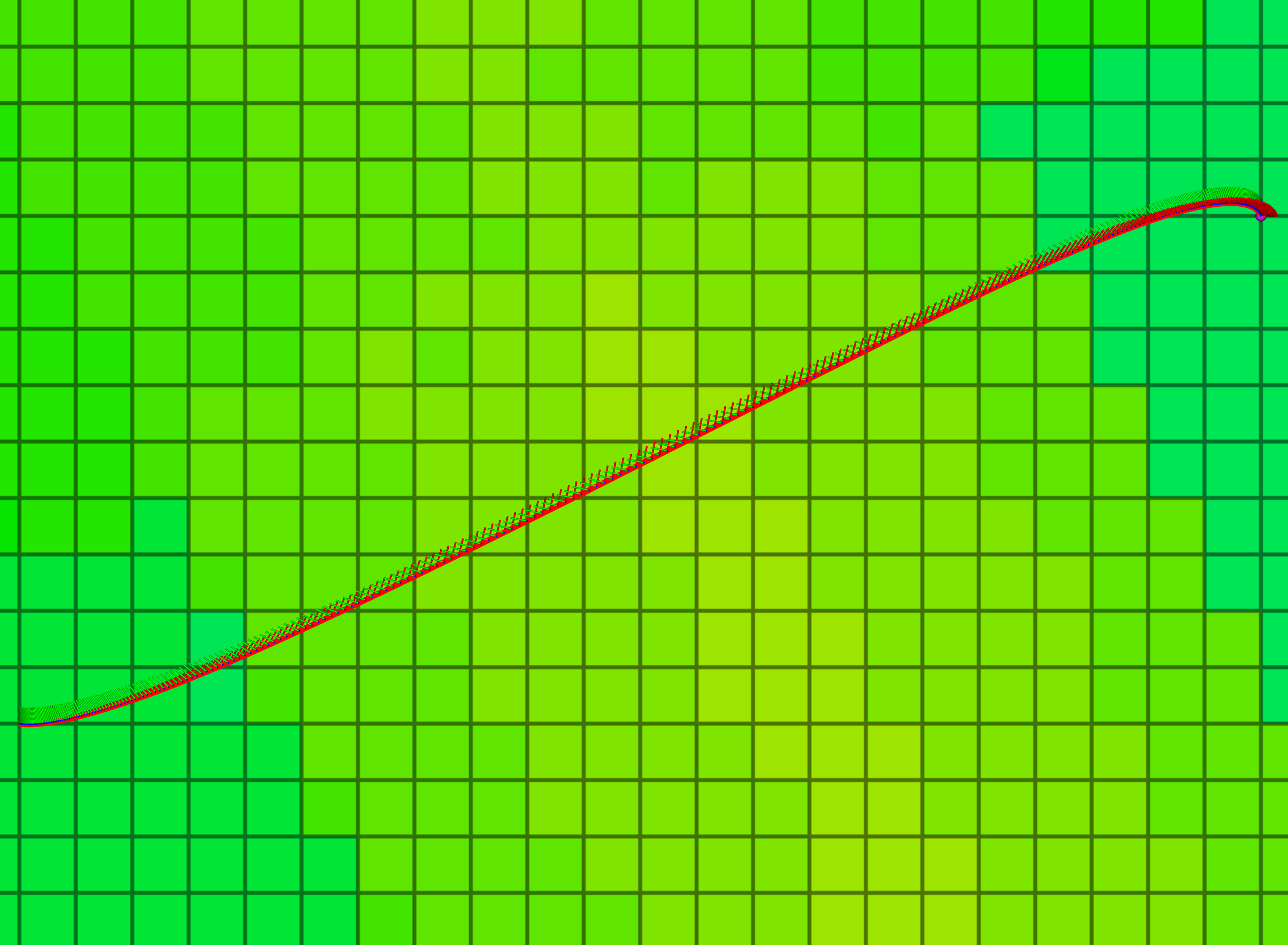}
  \vspace*{-1ex}
  \caption{Trajectory without visibility constraints. The optimized trajectory passes the higher part of the building with a single arc.}
  \label{fig:frankenforst-no-constraints}
  \vspace{-2ex}
\end{figure}

A more realistic example is depicted in \reffig{fig:frankenforst-long}.
The map contains a small village, where buildings block the line-of-sight between start and target poses.
As the start pose is close to an L-shaped building, the MAV has to fly away from the facade first (right side of \reffig{fig:frankenforst-long}) and perform a partial spiraling motion to gain altitude.
After passing the building through a cut-in between higher parts of the roof, the descent is smoothly distributed along the remaining trajectory.
In comparison to the planned path---which is also valid \wrt visibility constraints---the optimized trajectory can be flown at higher velocity since it does not contain sharp turns.
Thus, the optimized trajectory is less compact.
\reffig{fig:frankenforst-no-constraints} shows the optimized trajectory without constraints as a reference.
The corresponding angles between consecutive trajectory points for both examples are depicted in \reffig{fig:long_angles}.
It can be seen that without constraints, the trajectory goes up nearly vertical and then reduces the ascent angle nearly linearly until descending nearly vertical again.
The visibility constraints are violated for approximately \SI{75}{\percent} of the flight time, resulting in a large collision hazard.
With enabled visibility constraints the ascent and descent are within the maximum allowed band.

\begin{figure}[t]
  \centering
  \includegraphics[width=0.7\linewidth]{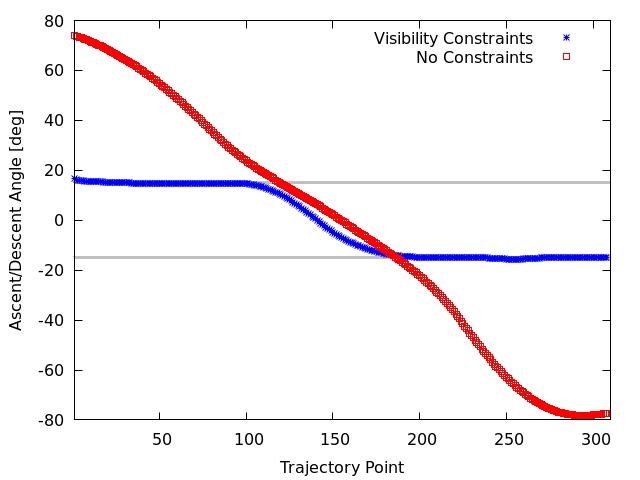}
  \vspace{-2ex}
  \caption{Angles for outdoor trajectory. Without constraints, the ascent and descent angles of the MAV trajectory, depicted in \reffig{fig:frankenforst-no-constraints}, change nearly linear from \SI{75}{\degree} to \SI{-80}{\degree} (red) caused by an arc-shaped trajectory over the building in the way. With enabled visibility constraints, the trajectory (see \reffig{fig:frankenforst-long}) is divided into an ascent, flight, and descent phase (blue). The angles stay within the band defined by the FoV of the sensor (gray lines).}
  \label{fig:long_angles}
\end{figure}

\begin{figure}[t]
  \setlength{\figureheight}{0.4\linewidth}
  \centerline{
  \begin{tikzpicture}[font=\sffamily]
    \node{
      \includegraphics[width=0.49\linewidth]{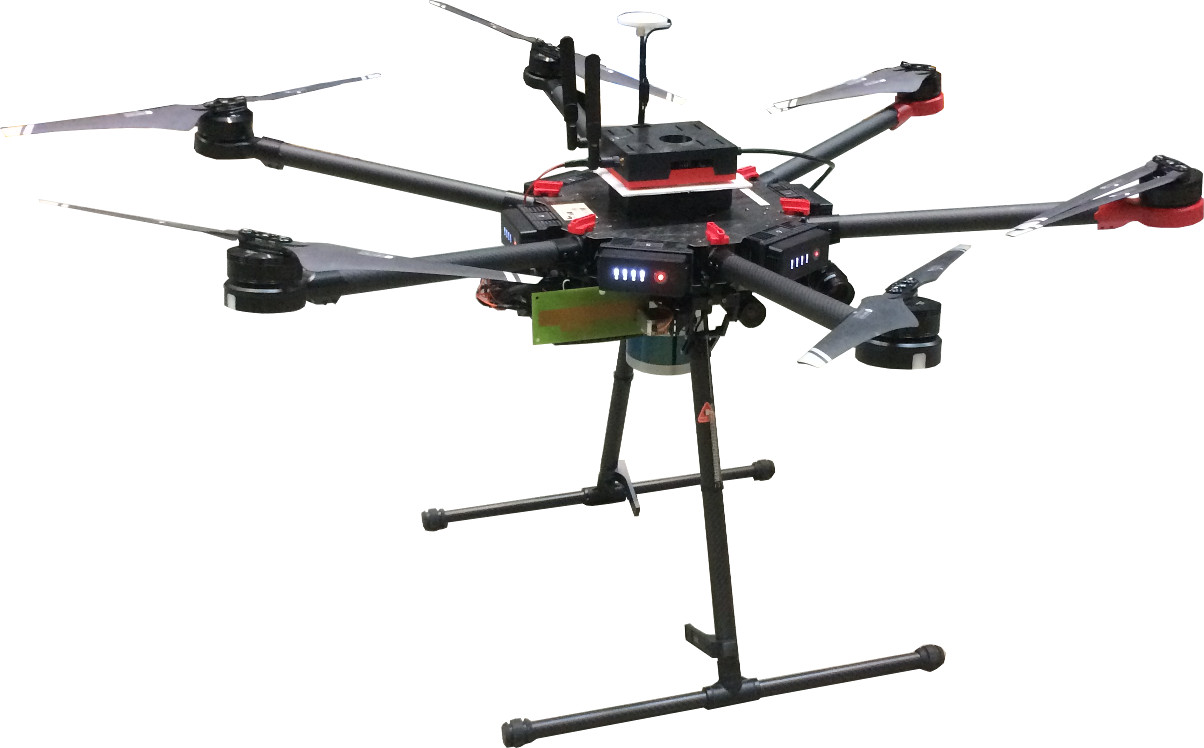}
    };
    \node[font=\footnotesize](label) at (-1.3,-0.8) {Laser scanner};
    \draw[thick,-,color=red](label)--(0.15,-0.0);
  \end{tikzpicture}
  \includegraphics[width=0.46\linewidth,clip,trim=0 0 100 100]{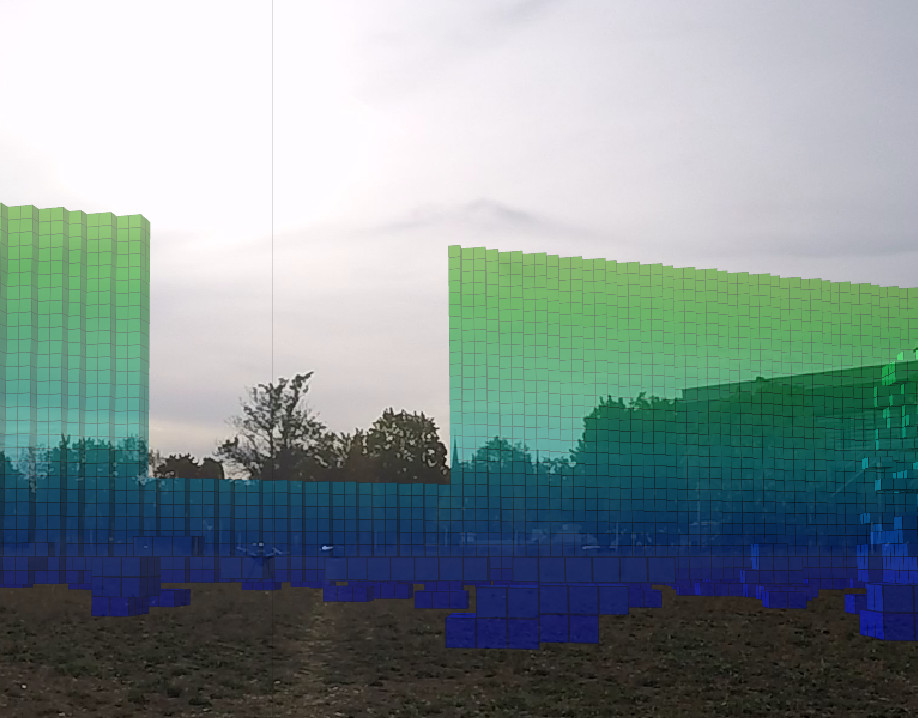}
}
  \vspace{-1ex}
  \caption{For our real-robot experiments, we employ a DJI Matrice 600 MAV (left). For obstacle avoidance, the MAV is equipped with a Velodyne Puck Lite 3D laser scanner with a vertical apex angle of \SI{30}{\degree}. The test environment is augmented with artificial obstacles (right).}
  \label{fig:m600}
  \vspace{-3ex}
\end{figure}

We evaluate the applicability of our approach for MAV control with our DJI Matrice 600 MAV~\cite{beul2018inventairy}, depicted in \reffig{fig:m600}.
In addition to outdoor experiments, we employ a hardware-in-the-loop (HIL) simulator provided by the MAV manufacturer DJI.
The optimized trajectories are executed by an MPC~\cite{beul2017icuas}.
Input to the controller are the next trajectory point position and velocity with \SI{10}{\hertz}.
The commands are sent open loop according to the calculated timings.
By interception prediction, the controller is able to track the trajectory accurately.

We report absolute trajectory errors (ATE) between optimized trajectories and the pose estimates of the MAV during simulated flight in \reftab{tab:ates}.
The ATEs are averaged over ten flights per example.
Spiral and Flight~1 are the trajectories depicted in \reffig{fig:spiral_up} and \reffig{fig:frankenforst-long}, respectively.
Flight~2 and Flight~3 are longer trajectories with different start and end points in the same map.
The MAV reaches velocities of up to \SI{2.43}{\meter\per\second} from an allowed maximum of \SI{3}{\meter\per\second} in the controller.
Thus, the resulting trajectories are within the dynamic limits of the MAV without slowing down the MAV too much.

\begin{figure}[t]
  \centering
  \setlength{\figureheight}{0.38\linewidth}
  \includegraphics[height=\figureheight]{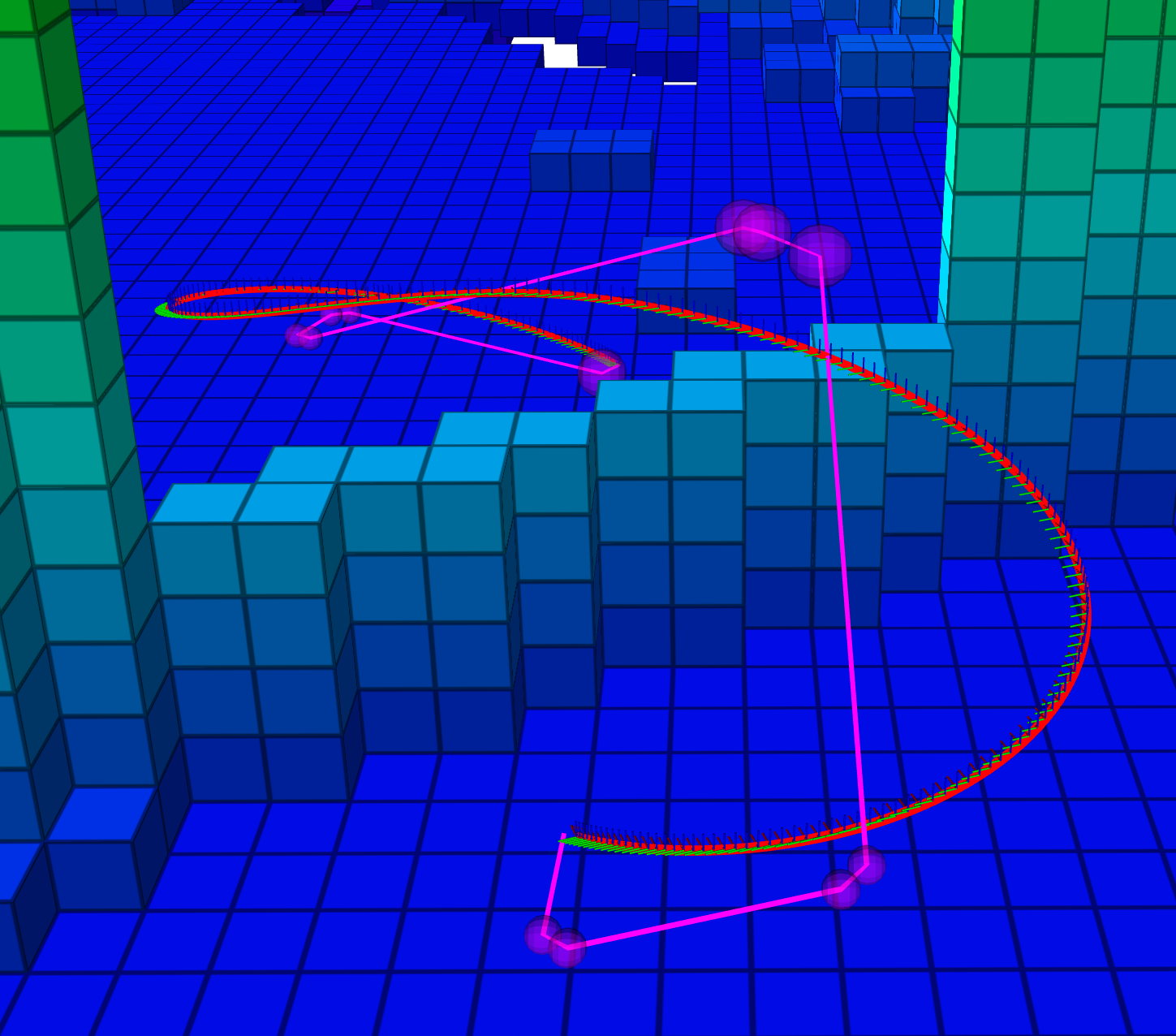}~
  \includegraphics[height=\figureheight]{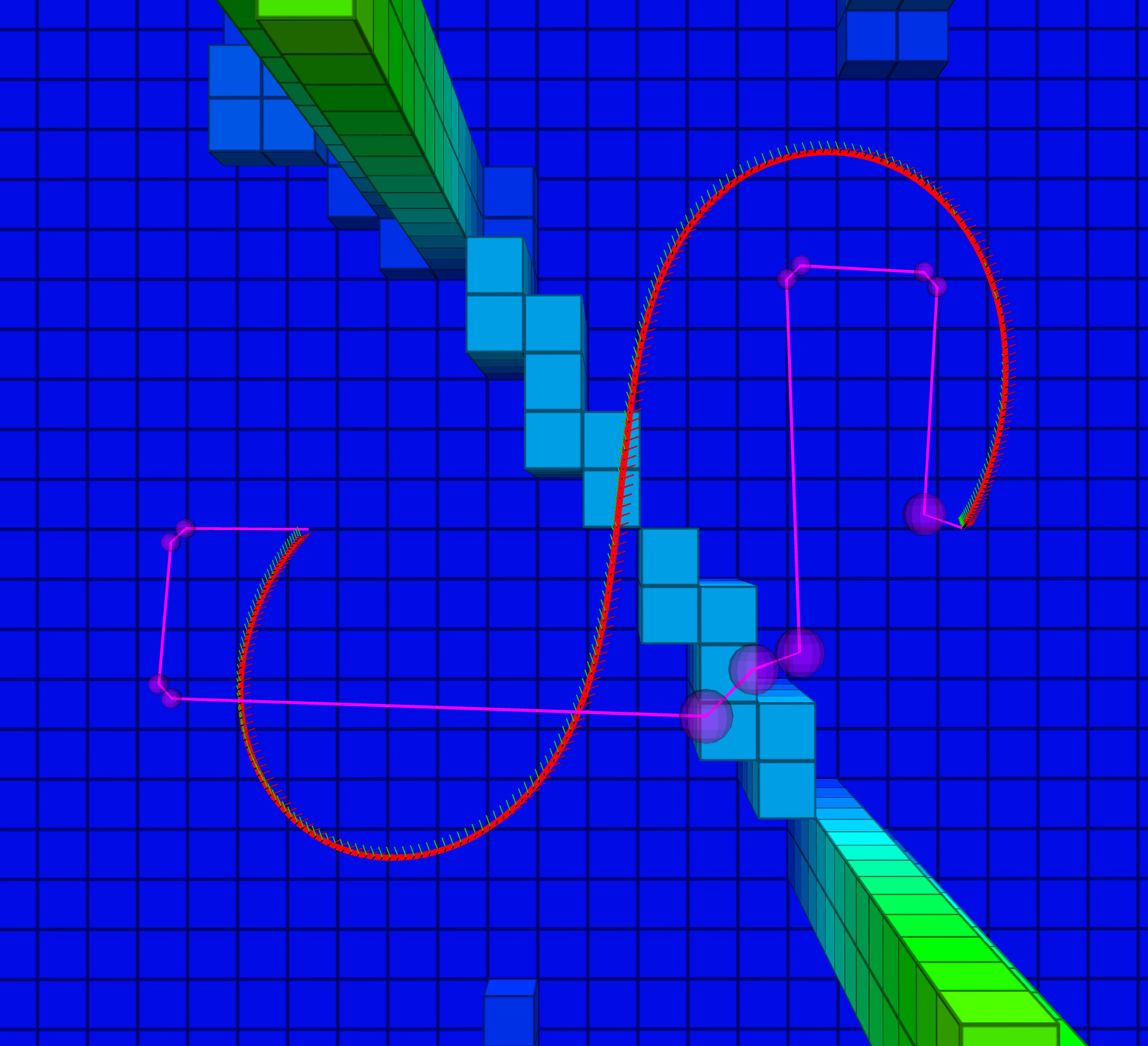}
  \vspace{-1ex}
  \caption{Example of a real-world experiment. Our MAV plans and optimizes a trajectory to overcome an artificial obstacle (flight from front/left to rear/right). The optimized trajectories are successfully followed by our Matrice 600 MAV. The depicted voxels have an edge length of \SI{1}{\meter}.}
  \label{fig:real-world}
\end{figure}

\begin{table}[t]
  \centering
  \caption{Simulation absolute trajectory errors (ATE).}
  \vspace*{-1ex} \begin{tabular}{lcccc}
    \toprule
    & Spiral & Flight 1 & Flight 2 & Flight 3\\
    \midrule
    ATE & 0.22 & 0.46 & 0.59 & 0.67\\
    RMSE & 0.14 & 0.30 & 0.34 & 0.37\\
    $v_{max}$ & 1.22 & 2.43 & 2.26 & 2.34\\
    \bottomrule
  \end{tabular}\\
  \vspace*{0.75ex}{\scriptsize ATE during trajectory execution (in m) averaged over ten flights. \\
  $v_{max}$ is the maximum velocity along the trajectory in m/s.}
  \label{tab:ates}
  \vspace{-4ex}
\end{table}

The outdoor experiments were performed in free-space augmented with artificial obstacles in the map.
\reffig{fig:real-world} shows an example with a high wall with an opening at a height of \SI{4}{\meter}.
To overcome the wall without violating the sensor FoV constraint, the MAV flies two connected partial spirals.
A second performed experiment includes an artificial wall with a uniform height of \SI{4}{\meter}.
In these experiments, the MAV plans and optimizes two qualitatively different trajectories---depending on the exact start condition.
The trajectories can either be of a shape comparable to the experiment with the opening or have a U-shape with roughly straight ascent and descent segments.
The third conducted experiment is an ascent in place similar to the spiral depicted in \reffig{fig:spiral_up}.

For state estimation in these experiments, we employ the onboard filter of the DJI flight control incorporating GPS and IMU measurements.
As no ground truth apart from this is available, the ATEs reported in \reftab{tab:ates-real-world} represent the trajectory tracking error based on the onboard state estimate.

\begin{figure}[t]
  \centering
  \includegraphics[width=0.8\linewidth,clip,trim=70 160 0 150]{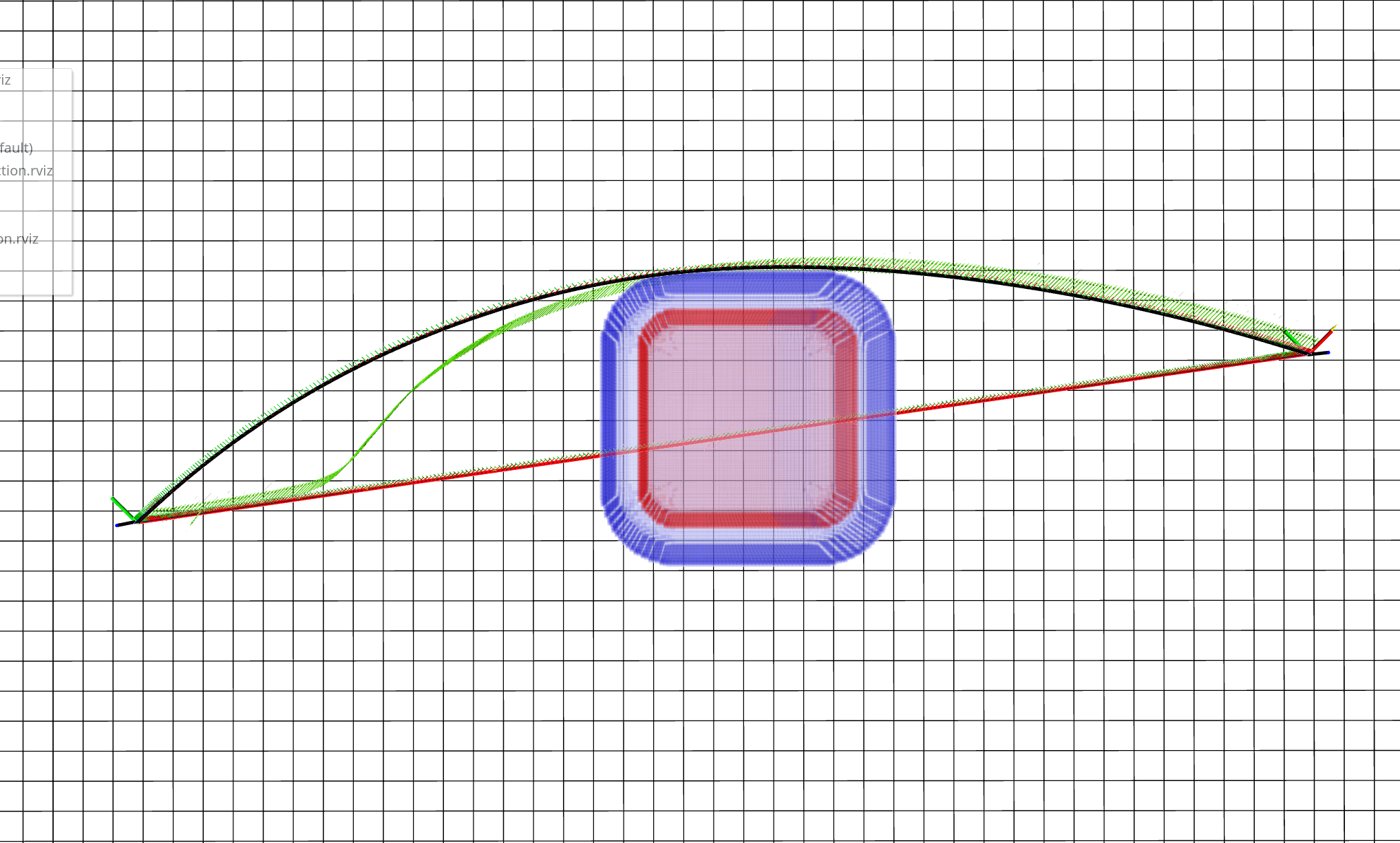}
  \vspace{-1ex}
  \caption{Continuous reoptimization allows for navigating around previously unknown obstacles. The red line depicts the initial trajectory; the green arrows depict the actual flown trajectory. The black line shows the resulting optimized trajectory if the obstacle is known in advance for reference. The obstacle is depicted by the isosurfaces for minimal and safe distance. The flight direction is from left to right.}
  \label{fig:reopt}
\end{figure}

\begin{table}[t]
  \centering
  \caption{Real-MAV ATEs during trajectory execution.}
  \vspace*{-1ex}
  \begin{tabular}{lccc}
    \toprule
    & Spiral & Wall & Opening \\
    \midrule
    ATE & 0.29 & 0.26 & 0.17 \\
    RMSE & 0.41 & 0.28 & 0.19 \\
    $v_{max}$ & 1.89 & 1.79 & 1.60 \\
    \bottomrule
  \end{tabular}\\
   \vspace*{0.75ex}{\scriptsize The ATEs are for individual flights. $v_{max}$ is the maximum\\  reached velocity along the trajectory in m/s.}
  \label{tab:ates-real-world}
  \vspace{-3ex}
\end{table}

Our tailored heuristic has the largest impact on the number of expanded nodes in the A* search, if the major difference between start and goal pose is a change in altitude.
For an ascent of \SI{7}{\meter} in place using an Euclidean distance heuristic results in \num{943505} node expansions.
Our FoV-aware heuristic reduces the number of expanded nodes to \num{285411}, which is approximately \SI{30}{\percent} of the baseline.
For the trajectories depicted in \reffig{fig:frankenforst-long} and \reffig{fig:real-world} the node expansions compared to the baseline are reduced to \SI{63}{percent} (\num{5907649} vs. \num{9443491} expansions) and \SI{88}{\percent} (\num{3766025} vs. \num{4255730} expansions), respectively.

We evaluate the reoptimization capabilities by placing an unknown cuboid obstacle of size \SI{4x4x4}{\meter} randomly with its center point within a corridor with radius \SI{1}{\meter} to the the line of sight between the start and goal pose of the MAV which is the initial best trajectory.
The scanner range of the MAV is reduced to \SI{15}{\meter} to avoid early detection of the obstacle.
\reffig{fig:reopt} shows the initial optimized trajectory and the actual flown trajectory with reoptimization for an example experiment.
With \num{10} iterations per reoptimization it took on average \SI{110}{\milli\second} depending on the remaining trajectory length, with a maximum of \SI{500}{\milli\second} for the full trajectory.
This is sufficient to find a feasible trajectory in a safe distance while approaching the obstacle.
Further reduction of this duration is possible with the multiresolution techniques from~\cite{nieuwenhuisen2016iros}, which we did not employ here.
The timings were measured on the MAV onboard PC.

The supplemental video shows footage of our outdoor experiments and results from the simulation experiments\footnote{\url{ais.uni-bonn.de/videos/ICRA_2019_Nieuwenhuisen}}.

\section{Conclusion}
Planning MAV trajectories imposes new challenges due to the ability for omnidirectional movement not only in the plane, but also in height.
Whereas the environment for ground vehicles can be covered relatively well with onboard obstacle sensors, the movement directions combined with a limited payload
prohibits complete and high-frequency coverage of the space around an MAV for many applications.
Our combined planning and trajectory optimization approach is capable to plan allocentric paths within the FoV of planar omnidirectional 3D sensors with a restricted apex angle in z-direction, \eg the popular Velodyne Puck Lite laser scanner.
The optimized trajectories are thus safe and dynamically feasible.
We showed that an MAV is able to follow these trajectories with an MPC in real-world experiments with our Matrice 600 MAV and in simulation employing a DJI flight control unit in the loop.

\balance
%%%%%%%%%%%%%%%%%%%%%%%%%%%%%%%%%%%%%%%%%%%%%%%%%%%%%%%%%%%%%%%%%%%%%%%%%%%%%%%%
\bibliographystyle{IEEEtran}
\bibliography{literature_references}

\end{document}